\documentclass[preprint,twoside,11pt]{article}

\usepackage{jmlr2e}

\usepackage[utf8]{inputenc}

\usepackage{abstract,amsmath,amssymb,latexsym}
\usepackage{enumitem,epsf}
\usepackage{algorithm}
\usepackage{algorithmic}

\newcommand{\norm}[1]{\left\lVert #1 \right\rVert}

\newcommand{\STAT}{\textrm{STAT}}
\newcommand{\Var}{\textrm{Var}}
\newcommand{\clamp}{\textrm{clamp}}
\newcommand{\ALG}{\textrm{ALG}}
\newcommand{\ORACLE}{\textrm{ORACLE}}

\usepackage{macros}

\jmlrheading{}{2020}{1-23}{}{}{}{}

\ShortHeadings{When Hardness of Approximation Meets Hardness of Learning}{}
\firstpageno{1}

\begin{document}

\title{When Hardness of Approximation Meets Hardness of Learning}

\author{\name Eran Malach \email eran.malach@mail.huji.ac.il \\
       \addr School of Computer Science\\
       The Hebrew University \\
       Jerusalem, Israel
       \AND
       \name Shai Shalev-Shwartz \email shais@cs.huji.ac.il \\
       \addr School of Computer Science\\
       The Hebrew University \\
       Jerusalem, Israel}

\editor{}

\maketitle

\begin{abstract}%
A supervised learning algorithm has access to a distribution of labeled examples, and needs to return a function (hypothesis) that correctly labels the examples. The hypothesis of the learner is taken from some fixed class of functions (e.g., linear classifiers, neural networks etc.). A failure of the learning algorithm can occur due to two possible reasons: wrong choice of hypothesis class (hardness of \textit{approximation}), or failure to find the best function within the hypothesis class (hardness of \textit{learning}). Although both approximation and learnability are important for the success of the algorithm, they are typically studied separately. In this work, we show a single hardness property that implies both hardness of approximation using linear classes and shallow networks, and hardness of learning using correlation queries and gradient-descent. This allows us to obtain new results on hardness of approximation and learnability of parity functions, DNF formulas and $AC^0$ circuits.
\end{abstract}

\begin{keywords}
Hardness of learning, approximation, statistical-queries, gradient-descent, neural networks
\end{keywords}

\maketitle

\section{Introduction}
Given a distribution $\cd$ over an instance space $\cx$ and a target classification function $f : \cx \to \{\pm 1\}$, let $f(\cd)$ be the distribution over $\cx \times \{\pm 1\}$ obtained by sampling $x \sim \cd$ and labeling it by $f(x)$. A learning algorithm, $\ALG$, has access to the distribution $f(\cd)$ via an oracle, $\ORACLE(f, \cd)$, and should output a hypothesis $h : \cx \to \reals$. The quality of $h$ is assessed by the expected loss function:
\[
L_{f(\cd)}(h) := \E_{(x,y) \sim f(\cd)}[ \ell(h(x),y)] ~,
\]
where $\ell : \reals \times \{\pm 1\} \to \reals_+$ is some loss function. We say that the learning is $\epsilon$-successful if $\ALG$ returns a function $\ALG(f,\cd)$ such that:
\[
\E[ L_{f(\cd)}(\ALG(f,\cd))] \le \epsilon ~,
\]
where the expectation is with respect to the randomness of the learning process. Of course, due to the well known \textit{no-free-lunch} theorem, we cannot expect any single algorithm to succeed in this objective for all choices of $(f,\cd)$. So, we must make some assumptions on the nature of the labeled distributions observed by the algorithm. We denote by $\ca$ the \textit{distribution family} (or \textit{assumption class}, as in \cite{kearns1994toward}), which is some set of pairs $(f,\cd)$, where $f$ is some function from $\cx$ to $\{\pm 1\}$ and $\cd$ is some distribution over $\cx$. We say that {\bf $\ALG$ is $\epsilon$-successful} on a distribution family $\ca$, if it $\epsilon$-succeeds on every $(f,\cd) \in \ca$, namely:
\[
\max_{(f,\cd) \in \ca} \E \left[ L_{f(\cd)}(\ALG(f,\cd)) \right] \le \epsilon
\]

The standard approach for understanding whether some algorithm is successful is using a \textit{decomposition of error}. Let $\ch$ be the class of functions that $\ALG$ can return (the \textit{hypothesis class}), and note that:
\begin{align*}
&\max_{(f,\cd) \in \ca} \E \left[ L_{f(\cd)} \left(\ALG(f,\cd) \right) \right] \\
&\le \underbrace{\max_{(f,\cd) \in \ca} \min_{h \in \ch} L_{f(\mathcal{D})}(h)}_{\textrm{\textbf{approximation} error}}
+ \underbrace{\max_{(f,\cd)\in \ca} \mean{L_{f(\mathcal{D})}(\ALG(f,\cd))} - \min_{h \in \ch}L_{f(\mathcal{D})}(h)}_{\textrm{\textbf{learning} error}}
\end{align*}
Similarly, it is easy to verify that
\begin{align*}
&\max\left\{\textrm{\textbf{approximation} error}~,~ \textrm{\textbf{learning} error}\right\} ~\le~ \max_{(f,\cd) \in \ca} \E \left[ L_{f(\cd)} \left(\ALG(f,\cd) \right) \right] 
\end{align*}
Therefore, a sufficient condition for $\ALG$ to  be $\epsilon$-successful is that both the approximation error and  learning error are at most $\epsilon/2$, and a necessary condition for $\ALG$ to be $\epsilon$-successful  is that both the approximation error and  learning error are at most $\epsilon$. 

In general, we say that $\ALG$ $\epsilon$-\textit{learns} $\ca$ if it achieves a learning error of at most $\epsilon$ (i.e., returns a hypothesis that is $\epsilon$-competitive with the best hypothesis in $\ch$), and we say that $\ch$ $\epsilon$-\textit{approximates} $\ca$ if $\ch$ has an approximation error of at most $\epsilon$ on $\ca$. The above inequalities show that in order for $\ALG$ to be successful, it must $\epsilon$-learns $\ca$ and at the same time, its hypothesis class must $\epsilon$-approximates $\ca$. However, the problems of \textit{learnability} and \textit{approximation} were typically studied separately in the literature of learning theory.

On the problem of \textit{learnability}, there is rich literature covering possibility and impossibility results in various settings of learning. These settings can be recovered by different choices of $\ca = \cf \times \cp$, where $\cf$ is some class of Boolean functions and $\cp$ is some class of distributions over $\cx$. The \textit{realizable} setting is given by assuming $\cf \subseteq \ch$, and the \textit{agnostic} setting is when $\cf$ is the class of all Boolean functions. In \textit{distribution-free} learning $\cp$ is the class of all distributions, while \textit{distribution-specific} learning assumes $\cp = \{\cd\}$, for some fixed distribution $\cd$. The literature of learning theory also considers various choices for the oracle $\ORACLE(f,\cd)$. The most common choice is the examples oracle, which gives the learner access to random examples sampled i.i.d. from $f(\cd)$. Other oracles calculate statistical queries (estimating $\E_{(\x,y) \sim f(\cd)} \psi(\x,y)$), membership queries (querying for labels of specific examples $\x \in \cx$) or return gradient estimations.

The question of \textit{approximation} has recently received a lot of attention in the machine learning community, with a growing number of works studying the limitations of linear classes, kernel methods and shallow neural networks, in terms of approximation capacity (see Section \ref{sec:hardness_approx}). These results often offer a separation property, showing that one hypothesis class (e.g., deep neural networks) is superior over another (e.g., shallow neural networks or linear classes). However, these questions are typically studied separately from the question of learnability.

In this work, we study the questions of learnability and approximation of general distribution families in a unified framework. We show that a single property, which we call the \textit{variance} of the distribution family, can be used for showing both \textit{hardness of approximation} and \textit{hardnes of learning} results. Specifically, we show: 1) hardness of approximating $\ca$ using any linear class with respect to a convex loss, 2) hardness of approximating some induced function using a shallow (depth-two) neural network with respect to a convex Lipschitz loss, 3) hardness of learning $\ca$ using correlation queries and 4) hardness of learning $\ca$ using gradient-descent. 

Applying our general results to some specific choices of $\ca$, we establish various novel results, in different settings of PAC learning:
\begin{enumerate}
\item Parities are hard to approximate using linear classes, under the uniform distribution.
\item The function $F(\x,\z) = \prod_i (x_i \vee z_i)$ is hard to approximate using depth-two neural networks, under the uniform distribution.
\item DNFs are hard to approximate using linear classes (\textit{distribution-free}).
\item The function $F(\x,\z) = \bigwedge_{i=1}^m \bigvee_{j=1}^{dm^2} (x_{ij} \wedge z_{ij})$ is hard to approximate using depth-two networks, for some fixed distribution.
\item Learning DNFs with correlation queries requires $2^{\Omega(n^{1/3})}$ queries (\textit{distribution-free}).
\item Learning DNFs with noisy gradient-descent requires $2^{\Omega(n^{1/3})}$ gradient steps (\textit{distribution-free}).
\end{enumerate}
We note that the approximation in 1-4 is with respect to a wide range of convex loss functions, and 6 is shown with respect to the hinge-loss.

These results expand our understanding of learnability and approximation of various classes and algorithms. Note that despite the fact that our setting is somewhat different than traditional PAC learning, our results apply in the standard settings of PAC learning (either  \textit{distribution-specific} or \textit{distribution-free} PAC learning). Importantly, they all follow from a single ``hardness'' property of the family $\ca$. We believe this framework can be used to better understand the power and limitations of various learning algorithms.

\section{Related Work}

In this section, we overview different studies covering results on the approximation capacity and learnability of various classes and algorithms used in machine learning. We focus on works that are directly related to the results shown in this paper.

\subsection{Hardness of Approximation}
\label{sec:hardness_approx}
\paragraph{Linear Classes}
The problem of understanding the expressive power of a given class of functions has attracted great interest in the learning theory and theoretical computer science community over the years. A primary class of functions that has been extensively studied in the context of machine learning is the class of linear functions over a fixed embedding (for example, in \cite{ben2002limitations, forster2006smallest, sherstov2008halfspace, razborov2010sign}). Notions such as margin complexity, dimension complexity and sign-rank were introduced in order to bound the minimal dimension and norm required to exactly express a class of functions using linear separators.

However, since the goal of the learner is to approximate a target function (e.g., PAC learning), and not to compute it exactly, showing hardness of exact expressivity often seems irrelevant from a machine learning perspective. Several recent studies show hardness of approximation results on linear classes \citep{allen2019can, allen2020backward, yehudai2019power, daniely2020learning}. These works demonstrate a separation between linear methods and neural networks: namely, they show families of functions that are hard to approximate using linear classes, but are learnable using neural networks. A recent work by \cite{kamath2020approximate} gives probabilistic variants of the dimension and margin complexity in order to show hardness results on approximation using linear classes and kernel methods. We show new results on hardness of approximation using linear classes, extending prior results and techniques.

\paragraph{Neural Networks} Another line of research that has gained a lot of attention in recent years focuses on the limitation of shallow neural networks, in terms of approximation power. The empirical success of deep neural networks sparked many questions regarding the advantage of using deep networks over shallow ones. The works of \cite{eldan2016power} and \cite{daniely2017depth} show examples of real valued functions that are hard to efficiently approximate using depth two (one hidden-layer) networks, but can be expressed using three layer networks, establishing a separation between these two classes of functions.
The work of \cite{martens2013representational} shows an example of a binary function (namely, the inner-product mod-2 function), that is hard to express using depth two networks. Other works study cases where the input dimension is fixed, and show an exponential gap in the expressive power between networks of growing depth  \citep{delalleau2011shallow,pascanu2013number,telgarsky2015representation,
telgarsky2016benefits,cohen2016expressive,raghu2017expressive,montufar2017notes,serra2018bounding,
hanin2019complexity,malach2019deeper}. A recent work by \cite{vardi2020neural} explores the relation between depth-separation results for neural networks, and hardness results in circuit complexity. We derive new results on hardness of approximation using shallow networks, which follow from the general hardness property introduced in the paper.

\subsection{Computational Hardness}

\paragraph{Computational Complexity} Since the early works in learning theory, understanding which problems can be learned efficiently (i.e., in polynomial time) has been a key question in the field. Various classes of interest, such as DNFs, boolean formulas, decision trees, boolean circuits and neural networks are known to be computationally hard to learn, in different settings of learning (see e.g., \cite{kearns1998efficient}). In the case of DNF formulas, the best known algorithm for learning DNFs, due to \cite{klivans2004learning}, runs in time $2^{\tilde{O}(n^{1/3})}$. A work by \cite{daniely2016complexity} shows that DNFs are computationally hard to learn, by reduction from the random variant of the K-SAT problem. However, this work does not yet establish a computational lower bound that matches the upper bound in \cite{klivans2004learning}, and assumes some non-standard hardness assumption. Using our framework we establish a lower bound of $2^{\Omega(n^{1/3})}$, for some restricted family of algorithms, on the complexity of learning DNF formulas.

\paragraph{Statistical Queries} In the general setting of PAC learning, the learner gets a sample of examples from the distribution, which is used in order to find a good hypothesis. Statistical Query (SQ) learning is a restriction of PAC learning, where instead of using a set of examples, the learner has access to estimates of statistical computations over the distribution, that are accurate up to some tolerance. This framework can be used to analyze a rich family of algorithms, showing hardness results on learning various problems from statistical-queries. Specifically, it was shown that parities, DNF formulas and decision trees cannot be learned efficiently using statistical-queries, under the uniform distribution \citep{kearns1998efficient,blum1994weakly,blum2003noise,goel2020superpolynomial}. In the case of DNF formulas, the number of queries required to learn this concept class under the uniform distribution is quasi-polynomial in the dimension (i.e., $n^{O(\log n)}$, and see \cite{blum1994weakly}). However, we are unaware of any work showing SQ lower-bounds on learning DNFs under general distributions, as we do in this work.

An interesting variant of statistical-query algorithms is algorithms that use only correlation statistical-queries (CSQ), namely --- queries of the form $\Mean{(\x,y) \sim \cd}{\phi(\x) y}$ \citep{feldman2008evolvability}. While in the distribution-specific setting, the CSQ and the SQ models are equivalent \citep{bshouty2002using}, in the distribution-independent setting this is not the case. It has been shown that conjunctions are hard to \textit{strongly} learn using CSQ in the distribution-independent setting \citep{feldman2011distribution,feldman2012complete}. We show hardness of \textit{weak} learning using CSQ, for a general choice of distribution families.

\paragraph{Gradient-based Learning} Another line of works shows hardness results for learning with gradient-based algorithms. The work of \cite{shalev2017failures} and the work of \cite{abbe2018provable} show that parities are hard to learn using gradient-based algorithms. The work of \cite{shamir2018distribution} shows distribution-specific hardness results on learning with gradient-based algorithms. These works essentially show that gradient-descent is ``stuck'' at a sub-optimal point. The work of \cite{safran2018spurious} shows a natural instance of learning neural network which suffers from spurious local minima. We show that DNF formulas are hard to learn using gradient-descent, using a generalization of the techniques used in previous works.

\section{Problem Setting}
\label{sec:problem_setting}

We now describe in more detail the general setting for the problem of learning families of labeled distributions. Let $\cx$ be our input space, where we typically assume that $\cx = \{\pm 1\}^n$. We define the following:
\begin{itemize}
\item A (labeled) \textbf{distributions family} $\ca$ is a set of pairs $(f, \cd)$, where $f$ is a function from $\cx$ to $\{\pm 1\}$, and $\cd$ is some distribution over $\cx$.
We denote by $f(\cd)$ the distribution over $\cx \times \{\pm 1\}$ of labeled examples $(\x,y)$, where $\x \sim \cd$ and $y = f(\x)$ (equivalently, $f(\cd)(\x,y) = \ind_{y = f(\x)} \cd(\x)$).
\item A \textbf{hypothesis class} $\ch$ is some class of functions from $\cx$ to $\reals$.
\end{itemize}

Throughout the paper, we analyze the approximation of the distribution family $\ca$ with respect to some loss function $\ell : \reals \times \{\pm 1\} \to \reals_+$. Some popular loss functions that we consider explicitly:
\begin{itemize}
\item Hinge-loss: $\ell^{\mathrm{hinge}}(\hat{y},y) := \max \{1-\hat{y}y,0 \}$.
\item Square-loss: $\ell^{\mathrm{sq}}(\hat{y},y) := \max (y-\hat{y})^2$.
\item Zero-one loss: $\ell^{\mathrm{0-1}}(\hat{y},y) = \ind\{\sign(\hat{y}) = y\}$.
\end{itemize}
All of our results hold for the hinge-loss, which is commonly used for learning classification problems. Most of our results apply for other loss functions as well. The exact assumptions on the loss functions for each result are detailed in the sequel.

 For some hypothesis $h \in \ch$, some pair $(f,\cd) \in \ca$, and some loss function $\ell$, we define the loss of $h$ with respect to $(f,\cd)$ to be:
\[
L_{f(\cd)}(h) = \E_{(\x,y) \sim f(\cd)}{\ell(h(\x),y)}
= \E_{\x \sim \cd}{\ell(h(\x),f(\x))}
\]
Our primary goal is to $\epsilon$-\textit{succeed} on the family $\ca$ using the hypothesis class $\ch$. Namely, for any choice of $(f,\cd) \in \ca$, given access to the distribution $f(\cd)$ (via sampling, statistical-queries or gradient computations), return some hypothesis $h \in \ch$ with  $\E \left[ L_{f(\cd)}(h) \right] \le \epsilon$. To understand whether or not it is possible to succeed on $\ca$ using $\ch$, there are two basic questions that we need to account for:
\begin{itemize}
\item \textbf{Approximation}: for every $(f,\cd) \in \ca$, show that there exists $h \in \ch$ with $L_{f(\cd)}(h) \le \epsilon$. In other words, we want to bound:
\[
\max_{(f,\cd) \in \ca} \min_{h \in \ch} L_{f(\cd)}(h)
\]
\item \textbf{Efficient Learnability}: show an algorithm s.t. for every $(f,\cd) \in \ca$, given access to $f(\cd)$, returns in time polynomial in $n, 1/\epsilon$ a hypothesis $h \in \ch$ with:
\[
\E \left[L_{f(\cd)}(h)\right]-\min_{\hat{h} \in \ch} L_{f(\cd)}(\hat{h}
) \le \epsilon
\]
\end{itemize}
Clearly, if $\ch$ cannot approximate $\ca$, then no algorithm, efficient or inefficient, can succeed on $\ca$ using $\ch$. However, even if $\ch$ can approximate $\ca$, it might not be efficiently learnable.

Before we move on, we give a few comments relating this setting to standard settings in learning theory:
\begin{remark}
A very common assumption in learning is \textbf{realizability}, namely --- assuming that for every $(f,\cd) \in \ca$, we have $f \in \ch$. When assuming realizability, the question of approximation becomes trivial, and we are left with the question of learnability. In our setting, we do not assume realizability using $\ch$, but we do assume that the family is realizable using some concept class (not necessarily the one that our algorithm uses).
\end{remark}

\begin{remark}
There are two common settings in learning theory: \textbf{distribution-free} and \textbf{distribution-specific} learning. In the distribution-free setting, we assume that $\ca = \cf \times \cp$, where $\cp$ is the class of all distributions over $\cx$, and $\cf$ is some class of boolean functions over $\cx$. In the distribution-specific setting, we assume that $\ca = \cf \times \{\cd\}$, where $\cd$ is some fixed distribution over $\cx$ (say, the uniform distribution). In a sense, we consider a setting that generalizes both distribution-free and distribution-specific learning.
\end{remark}

\section{Approximation and Learnability of Orthogonal Classes}
\label{sec:approx_orth}

We start by considering the simple case of a distribution-specific family of orthogonal functions. Fix some distribution $\cd$ over $\cx$. We define $\inner{f,g}_{\cd} = \Mean{\x \sim \cd}{f(\x)g(\x)}$, and let $\cf$ be some set of functions from $\cx$ to $\{\pm 1\}$ that are orthonormal with respect to $\inner{\cdot, \cdot}_{\cd}$. Namely, for every $f, g \in \cf$ we have $\inner{f,g}_{\cd} = \ind\{f = g\}$. For example, take $\cd$ to be the uniform distribution over $\cx = \{\pm 1\}^n$, and $\cf$ to be a set of parities, i.e. functions of the form $f_I(\x) = \prod_{i \in I} x_i$, for $I \subseteq [n]$. Then, we observe the orthonormal family $\ca = \cf \times \{\cd\}$.

First, we introduce the basic property of the family $\ca$ that will allow us to derive the various results shown in this section. Fix some function $\phi : \cx \to [-1,1]$, and observe the \textit{variance}\footnote{This is the variance per se only in the case where $\E \inner{f,\phi} = 0$, but we will refer to this quantity as variance in other cases as well.} of $\inner{f,\phi}_\cd$, over all the pairs $(f,\cd) \in \ca$:
\[
\Var(\ca,\phi) := \Mean{(f,\cd) \sim \ca}{\inner{f,\phi}_\cd^2}
\]
We can define the ``variance'' of $\ca$ by taking a supremum over all choices of $\phi$, namely:
\[
\Var(\ca) := \sup_{\norm{\phi}_\infty \le 1} \Var(\ca,\phi)
\]

Now, for the specific choice of orthonormal family $\ca$, we get that, using Parseval's identity:
\[
\Var(\ca, \phi) = \Mean{(f,\cd) \sim \ca}{\inner{f,\phi}_\cd^2}
= \frac{1}{\abs{\cf}} \sum_{f \in \cf} \inner{f,\phi}_{\cd}^2 \le \frac{\norm{\phi}_{\cd}^2}{\abs{\cf}} ~. 
\]
Since $\norm{\phi}_{\cd}^2 \le \|\phi\|_\infty^2$, it follows that
\[
\Var(\ca) \le \frac{1}{\abs{\cf}} ~.
\]
So, as $\abs{\cf}$ grows, $\Var(\ca)$ decreases. In fact, if we take $\cf$ to be all the parities over $\cx$, then $\Var(\ca)$ becomes exponentially small (namely, $2^{-n}$). We will next show how we can use $\Var(\ca)$ to bound the approximation and learnability of $\ca$ using various classes and algorithms.

\subsection{Approximation using Linear Classes}

Fix some embedding $\Psi : \cx \to [-1,1]^N$ and consider the family of linear predictors over this embedding: 
\[
\ch_\Psi = \{\x \mapsto \inner{\Psi(\x),\bw} ~:~ \norm{\bw}_2 \le B\}
\]
This is often a popular choice of a hypothesis class. We start by analyzing it's approximation capacity, with respect to some orthogonal family $\ca$. We rely on a very simple observation, that will be the key of the analysis in this section. Let $\ell$ be some convex loss function, satisfying $\ell(0,y) = \ell_0$ and $\ell'(0,y) = -y$. Fix some $(f,\cd) \in \ca$, and for every $\bw \in \reals^N$, define $L_{f(\cd)}(\bw) = L_{f(\cd)}(h_\bw)$, where $h_\bw(\x) = \inner{\Psi(\x),\bw}$. Since $\ell$ is convex, we get that $L_{f(\cd)}$ is convex as well. Therefore, for every $\bw \in \reals^N$ with $\norm{\bw}_2 \le B$ we have:
\begin{align*}
L_{f(\cd)}(\bw) &\ge L_{f(\cd)}(\Zero) + \inner{\nabla L_{f(\cd)}(\Zero), \bw} \\
&\ge^{C.S} L_{f(\cd)}(\Zero) -\norm{\nabla L_{f(\cd)}(\Zero)} \norm{\bw} \ge \ell_0 - B \norm{\nabla L_{f(\cd)}(\Zero)}
\end{align*}
This immediately gives a lower bound on the approximation of $\ca$ using $\ch_{\Psi}$:
\begin{equation}
\label{eq:lin_approx_bound}
\begin{split}
\E_{(f,\cd) \sim \ca} \min_{h \in \ch_\Psi} L_{f(\cd)}(h)
= \E_{(f,\cd) \sim \ca} \min_{\norm{\bw} \le B} L_{f(\cd)}(\bw)  \ge \ell_0 - B \E_{(f,\cd) \sim \ca} \norm{\nabla L_{f(\cd)}(\Zero)}
\end{split}
\end{equation}
So, upper bounding the average gradient norm, w.r.t. a random choice of $(f,\cd) \in \ca$, gives a lower bound on approximating $\ca$ with $\ch_\Psi$. Now, using our definition of $\Var(\ca)$ we get:
\begin{align*}
\E_{(f,\cd) \sim \ca} \left[\norm{\nabla L_{f(\cd)}(\Zero)}^2\right]
&= \E_{(f,\cd) \sim \ca} \left[ \sum_{i \in [N]} \left(\E_{\x \sim \cd}\ell'(0,f(\x)) \Psi(\x)_i\right)^2 \right] \\
&= \sum_{i \in [N]} \Mean{(f,\cd) \sim \ca}{\inner{\Psi_i,f}_{\cd}^2} \le N \cdot \Var(\ca)
\end{align*}
Using Jensen's inequality gives $\E_{(f,\cd) \sim \ca} \norm{\nabla L_{f(\cd)}(\Zero)} \le \sqrt{N} \sqrt{\Var(\ca)}$, and plugging in to Eq. (\ref{eq:lin_approx_bound}) we get:
\begin{equation}
\label{eq:lin_approx_lowerbound}
\max_{(f,\cd) \in \ca} \min_{h \in \ch_\Psi} L_{f(\cd)}(h)
\ge \E_{(f,\cd) \sim \ca} \min_{h \in \ch_\Psi} L_{f(\cd)}(h)
\ge \ell_0- B \sqrt{N} \sqrt{\Var(\ca)}
\end{equation}
The above result is in fact quite strong: it shows a bound on approximating the class $\ca$ using any choice of linear class (i.e., linear function over fixed embedding), and any convex loss functions (satisfying our mild assumptions). For example, it shows that any linear class $\ch_\Psi$ of polynomial size (with $B,N$ polynomial in $n$) cannot even \textit{weakly} approximate the family of parities over $\cx$. The loss of any linear class in this case will be effectively $\ell_0$, that is --- the loss of a constant-zero function. This extends the result of \cite{kamath2020approximate}, showing a similar result for the square-loss only.

\subsection{Approximation using Shallow Neural Networks}

The previous result shows a hardness of approximation, and hence a hardness of learning, of any family of orthogonal functions, using a linear hypothesis class. Specifically, we showed that approximating parities over $\cx$ is hard using any linear class. We now move to a more complex family of functions: depth-two (one hidden layer) neural networks. Given some activation $\sigma$, we define the class of depth-two networks by:
\[
\ch_{2\mathrm{NN}} = \left\lbrace\x \mapsto \sum_{i=1}^{k} u_i \sigma \left( \inner{\bw^{(i)}, \x} + b_i \right) ~:~ \norm{\bw^{(i)}}_2, \norm{\bu}_2, \norm{\bb}_2 \le R \right\rbrace
\]
It has been shown (e.g., \cite{shalev2017failures}) that $\ch_{2\mathrm{NN}}$ can implement parities over $\cx$. Therefore, together with Eq. (\ref{eq:lin_approx_lowerbound}) shown previously, this gives a \textit{strong} separation between the class of depth-two networks and \textbf{any} linear class: while parities can be implemented exactly by depth-two networks, using a linear class they cannot even be approximated beyond a trivial hypothesis.

However, we can leverage the previous results to construct a function that \textbf{cannot} be approximated using a depth-two network. Our construction will be as follows: let $\cz \subseteq \{\pm 1\}^n$ be some subspace, and define some bijection $\varphi : \cz \to \ca$. Observe the function $F : \cx \times \cz \to \{\pm 1\}$ defined as $F(\x,\z) = \varphi(\z)_1(\x)$, and the distribution $\cd'$ over $\cx \times \cz$ where $(\x,\z) \sim \cd'$ is given by sampling $(f,\cd) \sim \ca$ uniformly, sampling $\x \sim \cd$ and setting $\z = \varphi^{-1}(f,\cd)$. We call the function $F$ the induced function and the distribution $\cd'$ the induced distribution.

Following our general definition of $\ch_{2\mathrm{NN}}$, we define a depth-two neural-network over $\cx \times \cz$ by:
\[
g(\x, \z) = \sum_{i=1}^k u_i \sigma\left(\inner{\bw^{(i)}, \x} + \inner{\bv^{(i)},\z} + b_i\right), ~ \norm{\bw^{(i)}}_2, \norm{\bv^{(i)}}_2, \norm{\bu}_2, \norm{\bb}_2 \le R
\]
In this case, we show the following result:
\begin{theorem}
\label{thm:shallow_approx_lowerbound}
Fix some distribution family $\ca$. Let $\ell$ be a 1-Lipschitz convex loss satisfying $\ell(0,y) = \ell_0$, $\ell'(0,y) = -y$. Then, every depth-two neural network $g : \cx \times \cz \to \reals$ with any $1$-Lipschitz activation satisfies:
\[
L_{F(\cd')}(g) \ge \ell_0 - 6\sqrt{k}R^{2}n^{5/6} \Var(\ca)^{1/3}
\]
where $F$ and $\cd'$ are induced from $\ca$.
\end{theorem}

We will start by showing this result in the case where $\bv^{(i)}$ takes discrete values:

\begin{lemma}
\label{lem:shallow_approx_integer}
Assume that there exists $\Delta > 0$ such that $v_j^{(i)} \in \Delta \integers := \{\Delta \cdot z ~:~ z \in \integers \}$ for every $i,j$ and $\norm{\bu^{(i)}},\norm{\bw^{(i)}},\norm{\bv^{(i)}}, \norm{\bb} < R$. Then:
\[
L_{F(\cd')}(g) = \mean{\ell(g(\x,\z),F(\x,\z))} \ge \ell_0 - 3\sqrt{2k }R^{5/2} n^{3/4} \sqrt{\frac{\Var(\ca)}{\Delta}}
\]
\end{lemma}

\begin{proof}
The key for proving Lemma \ref{lem:shallow_approx_integer} is to reduce the problem of approximating $(F,\cd')$ using a shallow network to the problem of approximating $\ca$ using some linear class. That is, we fix some $\bw^{(i)}, \bv^{(i)} \in \reals$, and find some $\Psi : \cx \to [-1,1]^N$ such that $g(\x,\z) = \inner{\Psi(\x), \bu(\z)}$, for some $\bu(\z) \in \reals^N$. 

To get this reduction, we observe that since $\bv^{(i)}$ is discrete, $\inner{\z,\bv^{(i)}}$ can take only a finite number of values. In fact, we have $\inner{\z,\bv^{(i)}} \in [-\sqrt{n}R,\sqrt{n}R] \cap \Delta \integers$.
Indeed, fix some $i$ and we have $\frac{1}{\Delta} \bv^{(i)} \in \integers^n$, and since $\z \in \integers^n$ we have $\frac{1}{\Delta} \inner{\bv^{(i)}, \z} = \inner{\frac{1}{\Delta} \bv^{(i)}, \z} \in \integers$. So, we can map $\x$ to $\sigma(\inner{\bw^{(i)},\x} + j + b_i)$, for all choices of $j \in [-\sqrt{n} R, \sqrt{n} R] \cap \Delta \integers$, and get an embedding that satisfies our requirement. That is, we use the fact that $\inner{\bw^{(i)},\z}$ ``collapses'' to a small number of values to remove the dependence of $g(\x,\z)$ in the exact value of $\z$.

To show this formally, for every $\z \in \cz$ denote $j(\z) = \inner{\bv^{(i)},\z}$, and so from what we showed $j(\z) \in [-R\sqrt{n},R\sqrt{n}] \cap \Delta\integers$. 
Define $\Psi_{i,j}(\x) = \frac{1}{3R\sqrt{n}}\sigma\left(\inner{\bw^{(i)},\x} + j + b_i\right)$ for every $i \in [k]$ and $j \in [-R\sqrt{n},R\sqrt{n}]\cap \Delta\integers$, and note that:
\[
\abs{\Psi_{i,j}(\x)} \le \frac{1}{3R\sqrt{n}}\abs{\inner{\bw^{(i)},\x} + j + b_i}
\le \frac{1}{3R\sqrt{n}} \left(\norm{\bw^{(i)}}\norm{\x} + \abs{j} + \abs{b_i} \right) \le 1
\]
Notice that $\abs{\left[-R\sqrt{n}, R\sqrt{n} \right]\cap \Delta \integers} \le 2\left\lfloor \frac{R\sqrt{n}}{\Delta} \right \rfloor$, and so there are at most $2 \left\lfloor \frac{R\sqrt{n}}{\Delta} \right \rfloor$ choices for $j$.
Denote $N := 2k\lfloor \frac{R\sqrt{n}}{\Delta} \rfloor$ and let $\Psi : \mathcal{X} \to [-1,1]^N$ defined as $\Psi(\x) = [\Psi_{i,j}(\x)]_{i,j}$ (in vector form). Denote $B = 3R^2\sqrt{n}$, and from Eq. (\ref{eq:lin_approx_lowerbound}):
\begin{align*}
\E_{\z \sim U(\cz)} \left[\min_{\norm{\hat{\bu}}\le B} \E_{(\x,y) \sim \varphi(\z)} \ell(\inner{\hat{\bu},\Psi(\x)}, y)\right] 
&= \E_{\z} \left[\min_{h \in \mathcal{H}_\Psi^B} L_{\varphi(\z)}(h)\right] \\
&= \E_{(f,\cd) \sim \ca} \left[\min_{h \in \mathcal{H}_\Psi^B} L_{f(\cd)}(h)\right]
\ge \ell_0-B \sqrt{N} \sqrt{\Var(\ca)}
\end{align*}
Notice that $g(\x,\z) = \sum_{i=1}^k 3R\sqrt{n}u_i \Psi_{i,j(\z)}(\x) = \inner{\bu(\z), \Psi(\x)}$ where: 
\[
\bu(\z)_{i,j} = \begin{cases} 3R\sqrt{n}u_i & j = j(\z) \\ 0 & j \ne j(\z) \end{cases}\]
Since $\norm{\bu(\z)} \le 3R \sqrt{n} \norm{\bu} \le B$ we get that:
\begin{align*}
\E_{\z \sim U(\cz), \x \sim \varphi(\z)} \left[\ell(g(\x,\z), F(\x,\z))\right]
&= \E_{\z \sim U(\cz), \x \sim \varphi(\z)} \left[\ell(\inner{\bu(\z), \Psi(\x)}, F(\x,\z))\right] \\
&\ge \E_{\z \sim U(\cz)} \left[\min_{\norm{\hat{\bu}}\le B} \E_{(\x,y) \sim \varphi(\z)}\ell(\inner{\hat{\bu},\Psi(\x)}, y)\right] \\
&\ge \ell_0-B\sqrt{N}\sqrt{\Var(\ca)}
\end{align*}
\end{proof}

%
%

Now, to prove Theorem \ref{thm:shallow_approx_lowerbound}, we use the fact that a network with arbitrary (bounded) weights can be approximated by a network with discrete weights.

\begin{proof} of Theorem \ref{thm:shallow_approx_lowerbound}.
Fix some $\Delta \in (0,1)$, and let $\hat{\bv}^{(i)} = \Delta \left \lfloor \frac{1}{\Delta} \bv^{(i)} \right\rfloor \in \Delta \integers^n$, where $\left\lfloor \cdot\right\rfloor$ is taken element-wise.
Notice that for every $j$ we have:
\[
\abs{v^{(i)}_j - \hat{v}^{(i)}_j} = \abs{v^{(i)}_j - \Delta \left\lfloor \frac{1}{\Delta} v^{(i)}_j \right \rfloor} = \Delta \abs{\frac{1}{\Delta} v^{(i)}_j - \left\lfloor \frac{1}{\Delta} v^{(i)}_j \right \rfloor} \le \Delta
\]
Observe the following neural network:
\[
\hat{g}(\x, \z) = \sum_{i=1}^k u_i \sigma\left(\inner{\bw^{(i)}, \x} + \inner{\hat{\bv}^{(i)},\z} + b_i\right)
\]
For every $\x, \z \in \mathcal{X}$, using Cauchy-Schwartz inequality, and the fact that $\sigma$ is $1$-Lipchitz:
\begin{align*}
\abs{g(\x,\z) - \hat{g}(\x,\z)} &\le \norm{\bu} \sqrt{ \sum_{i=1}^k \abs{\sigma\left(\inner{\bw^{(i)}, \x} + \inner{\bv^{(i)},\z} + b_i\right) - \sigma\left(\inner{\bw^{(i)}, \x} + \inner{\hat{\bv}^{(i)},\z} + b_i\right)}^2} \\
&\le \norm{\bu} \sqrt{ \sum_{i=1}^k \abs{\inner{\bv^{(i)},\z} -  \inner{\hat{\bv}^{(i)},\z}}^2} \\
&\le \norm{\bu} \sqrt{ \sum_{i=1}^k \norm{\bv^{(i)}- \hat{\bv}^{(i)}}^2\norm{\z}^2} \le R \sqrt{k} \Delta n
\end{align*}
Now, by Lemma \ref{lem:shallow_approx_integer} we have:
\[
L_{F(\cd')}(\hat{g}) \ge \ell_0-3\sqrt{2k }R^{5/2} n^{3/4}  \sqrt{\frac{\Var(\ca)}{\Delta}}
\]
And using the fact that $\ell$ is $1$-Lipschitz we get:
\begin{align*}
L_{F(\cd')}(g) &= \mean{\ell(g(\x,\z),F(\x,\z))} \\
&\ge \mean{\ell(\hat{g}(\x,\z),F(\x,\z))} - \mean{\abs{\ell(g(\x,\z),F(\x,\z))-\ell(\hat{g}(\x,\z),F(\x,\z))}} \\
&\ge L_{F(\cd')}(\hat{g}) - \mean{\abs{g(\x,\z)-\hat{g}(\x,\z)}} 
\ge \ell_0-3\sqrt{2k }R^{5/2} n^{3/4}  \sqrt{\frac{\Var(\ca)}{\Delta}} - R \sqrt{k} \Delta n
\end{align*}
This is true for any $\Delta > 0$, so we choose $\Delta = \frac{3^{2/3} 2^{1/3} R}{n^{1/6}}\Var(\ca)^{1/3}$ and we get:
\[
L_{F(\cd')}(g) \ge \ell_0 - 2^{4/3}3^{2/3}\sqrt{k}R^{2}n^{5/6}\Var(\ca)^{1/3}
\]
\end{proof}

\subsubsection{Hardness of approximation of inner-product mod 2}
We now interpret the result of Theorem \ref{thm:shallow_approx_lowerbound} for the case where $\ca$ is the family of parities over $\cx$ with respect to the uniform distribution. In this case, we can define $\cz = \cx = \{\pm 1\}^n$ and define $\varphi : \cz \to \ca$ such that $\varphi(\z) = (f_\z,\cd)$, where $f_\z$ is the parity such that $f_\z(\x) = \prod_{i \in [n], z_i = -1} x_i$. We can write the induced function as $F(\x,\z) = \prod_{i \in [n], z_i=-1} x_i = \prod_{i \in [n]} (x_i \vee z_i)$, and the induced distribution $\cd'$ is simply the uniform distribution over $\cx \times \cx$. Using Theorem \ref{thm:shallow_approx_lowerbound} and the fact that $\Var(\ca) = 2^{-n}$ we get the following:

\begin{corollary}
Let $F(\x,\z) = \prod_{i \in [n]}(x_i \vee z_i)$, and let $\cd'$ be the uniform distribution over $\{\pm 1\}^n \times \{\pm 1\}^n$.
Let $\ell$ be a 1-Lipschitz convex loss satisfying $\ell(0,y) = \ell_0$ and $\ell'(0,y) = -y$. Then, any polynomial-size network with polynomial weights and $1$-Lipschitz activation, cannot (weakly) approximate $F$ with respect to $\cd'$ and the loss $\ell$.
\end{corollary}

We note that $F$ is similar to the inner-product mod-2 function, that has been shown to be hard to implement efficiently using depth-two networks \citep{martens2013representational}. Our result shows that this function is hard to even approximate, using a polynomial-size depth-two network, under any convex loss satisfying our assumptions. Notice that $F$ \textbf{can} be implemented using a depth-three network, and so this result gives a strong separation between the classes of depth-two and depth-three networks (of polynomial size).

\subsection{Hardness of Learning with Correlation Queries}
So far, we showed hardness of approximation results for linear classes and shallow (depth-two) networks. This motivates the use of algorithms that learn more complex hypothesis classes (for example, depth-three networks). We now give hardness results that are independent of the hypothesis class, but rather focus on the learning algorithm. We show restrictions on specific classes of algorithms, for learning families $\ca$ with small $\Var(\ca)$. While such results are well-known in the case of orthogonal classes, we introduce them here in a fashion that allows us to generalize such results to more general families of distributions.

First, we consider learnability using statistical-query algorithms. A statistical-query algorithm has access to an oracle $\STAT(f,\cd)$ which, given a query $\psi : \cx \times \{\pm 1\} \to [-1,1]$, returns a response $v$ such that $\abs{\E_{\x \sim \cd} \psi(\x,f(\x))-v} \le \tau$, for some tolerance parameter $\tau > 0$. Specifically, we focus on algorithms that use only correlation queries, i.e. queries of the form $\psi(\x,y) = y\phi(\x)$ for some $\phi : \cx \to [-1,1]$.

We begin with the following key lemma:

\begin{lemma}
\label{lem:corr_count}
Fix some family $\ca$. Let $\phi : \cx \to [-1,1]$ be some function, and let $\ca_\phi \subseteq \ca$ be the subset of pairs $(f,\cd)$ such that $\abs{\inner{f,\phi}_\cd} = \abs{\E_{(\x,y) \sim f(\cd)} \left[y \phi(\x)\right]} \ge \tau$. Then:
\[
\abs{\ca_\phi} \le \frac{\Var(\ca)}{\tau^2}\abs{\ca}
\]
\end{lemma}

\begin{proof}
By definition of $\Var(\ca)$ we have that:
\[
\E_{(f,\cd) \sim \ca} \left(\E_{(\x,y) \sim f(\cd)} \left[y \phi(\x)\right]\right)^2 = \E_{(f,\cd) \sim \ca} \left[\inner{f,\phi}^2_\cd\right] \le \Var(\ca)
\]
By definition of $\ca_\phi$ we get:
\[
\E_{(f,\cd) \sim \ca} \left(\E_{(\x,y) \sim f(\cd)} \left[y \phi(\x)\right]\right)^2 \ge \frac{1}{\abs{\ca}} \sum_{(f,\cd) \in \ca_\phi} \left(\E_{(\x,y) \sim f(\cd)} \left[y \phi(\x)\right]\right)^2 \ge \frac{\abs{\ca_\phi}}{\abs{\ca}} \tau^2
\]
And so $\abs{\ca_\phi} \le \frac{\Var(\ca)}{\tau^2}\abs{\ca}$.
\end{proof}

Using this lemma, we can show the following hardness result:
\begin{theorem}
\label{thm:sq_hardness}
Let $\ell \in \{\ell^{\mathrm{hinge}}, \ell^{\mathrm{sq}}, \ell^{\mathrm{0-1}}\}$. Fix some family $\ca$. Then, for any $\tau > 0$, any statistical-query algorithm that makes only correlation queries needs to make at least $\frac{\tau^2}{\Var(\ca)}-1$ queries of tolerance $\tau$ to achieve loss $< a_\ell-b_\ell \tau$, for some universal constants $a_\ell, b_\ell>0$ that depend on the loss function.
\end{theorem}
\begin{proof}
Fix some family $\ca$, and some correlation query $\psi(\x,y) = y\phi(\x)$. Define $\ca_\phi$ as in Lemma \ref{lem:corr_count}. Note that for any $(f,\cd) \notin \ca_\phi$, the oracle $\STAT(f,\cd)$ can return $0$ on the query $\psi$. Assume a statistical query algorithm uses $q$ queries, and then returns some hypothesis $h : \cx \to \reals$. Then, the number of pairs $(f,\cd)$ that are not consistent with an oracle that outputs $0$ on all queries is at most $q \cdot \sup_{\phi}\abs{\ca_\phi} \le \frac{q \Var(\ca)}{\tau^2} \abs{\ca}$. So, if $q < \frac{\tau^2}{\Var(\ca)}-1$, we get that there are at least $\frac{\Var(\ca)}{\tau^2} \abs{\ca}$ pairs $(f,\cd)$ that are consistent with all the responses of the oracle. Let $\tilde{h} : \cx \to [-1,1]$ such that if $\ell$ is the zero-one loss, then $\tilde{h}(\x) = \sign(h(\x))$ and otherwise $\tilde{h}(\x) = \clamp(h(\x))$, where  $\clamp(\hat{y}) = \begin{cases} \hat{y} & \hat{y} \in [-1,1] \\ \sign(\hat{y}) & o.w \end{cases}$.
Either way, from Lemma \ref{lem:corr_count}, there are at most $\frac{\Var(\ca)}{\tau^2} \abs{\ca}$ pairs $(f,\cd)$ for which
$\abs{\inner{f,\tilde{h}}_\cd} \ge \tau$. All in all, there is some $(f,\cd) \in \ca$ that is consistent with all the responses of the oracle, and for which $\abs{\inner{f,\tilde{h}}_\cd} \le \tau$.
In this case, we get that, for some $a_\ell,b_\ell$:
\[
L_{f(\cd)}(h) = \Mean{(\x,y) \sim f(\cd)}{\ell(h(\x), y)}
\ge \Mean{(\x,y) \sim f(\cd)}{\ell(\tilde{h}(\x), y)}
\ge a_\ell- b_\ell\inner{f,\tilde{h}}_\cd \ge a_\ell-b_\ell\tau
\]
The first inequality holds for all choices of $\ell$, by definition of $\tilde{h}$. We show the second inequality separately for the different loss functions:
\begin{itemize}
\item If $\ell = \ell^{\mathrm{hinge}}$ then since $\tilde{h}(\x) \in [-1,1]$ we have $\ell(\tilde{h}(\x),y) = 1-y\tilde{h}(\x)$.
\item If $\ell = \ell^{\mathrm{sq}}$ then we have: $\ell(\tilde{h}(\x),y) = \tilde{h}(\x)^2 - 2\tilde{h}(\x)y + y^2 \ge 1- 2 \tilde{h}(\x) y$.
\item If $\ell = \ell^{\mathrm{0-1}}$ then we have: $\ell(\tilde{h}(\x),y) = \ind\{\tilde{h}(\x) = y\} = \frac{1}{2} - \frac{1}{2}\tilde{h}(\x)y$.
\end{itemize}
\end{proof}

\begin{remark}
Observe that in the case where $\ca$ is a distribution-specific family, any statistical-query algorithm can be modified to use only correlation queries, as shown in a work by \cite{bshouty2002using}. However, this is not true for general families of distributions.
\end{remark}

When $\ca$ is a family of parities with respect to the uniform distribution, Theorem \ref{thm:sq_hardness} along with the fact that $\Var(\ca) = 2^{-n}$ recovers the well-known result on hardness of learning parities using statistical-queries \citep{kearns1998efficient}.

\subsection{Hardness of Learning with Gradient-Descent}

We showed that families $\ca$ with small $\Var(\ca)$ are hard to learn using correlation queries. We now turn to analyze a specific popular algorithm: the gradient-descent algorithm. In this part we take the loss function to be the hinge-loss, so $\ell = \ell^{\mathrm{hinge}}$. 

Observe the following formulation of gradient-descent: let $\ch$ be a parametric hypothesis class, such that $\ch = \{h_\bw ~:~ \bw \in \reals^{N} \}$. In the gradient-descent algorithm, we initialize $\bw_0$ (possibly randomly), and perform the following updates:
\begin{equation}
\bw_t = \bw_{t-1} - \eta \nabla_{\bw} L_{f(\cd)}(h_{\bw_{t-1}})
\tag{GD Update}
\end{equation}

However, assuming we have access to the exact value of $\nabla_\bw L_{f(\cd)}$ is often unrealistic. For example, when running gradient-descent on a machine with bounded precision, we can expect the value of the gradient to be accurate only up to some fixed precision. So, we consider instead a variant of the gradient-descent algorithm which has access to gradients that are accurate up to a fixed precision $\Delta > 0$. The \textit{$\Delta$-approximate gradient-descent} algorithm performs the following updates:
\begin{equation}
\label{eq:approx_gd}
\bw_t = \bw_{t-1} - \eta \bv_{t-1}
\tag{Approx. GD Update}
\end{equation}
with $\bv_t \in \Delta \integers^N$ (where $\Delta \integers := \{\Delta z ~:~ z \in \integers\}$) satisfying $\norm{\bv_t - \nabla_\bw L_{f(\cd)}(h_{\bw_{t-1}})}_\infty \le \frac{\Delta}{2}$ (i.e., $\bv_t$ is the result of rounding $\nabla_\bw L_{f(\cd)}(h_{\bw_{t-1}})$ to $\Delta \integers^N$).

Notice that if the gradients are smaller than the machine precision, they will be rounded to zero, and so the gradient-descent algorithm will be ``stuck''. We show that if $\Var(\ca)$ is small, then for most choices of $(f,\cd) \in \ca$ the initial gradient will indeed be extremely small. The key for showing this is the following lemma:

\begin{lemma}
\label{lem:small_gradient}
Fix some $\bw \in \reals^N$ satisfying $\abs{h_{\bw}(\x)} \le 1$ and $\norm{\nabla_\bw h_{\bw}(\x)} \le B$, for every $\x \in \cx$. Then:
\[
\E_{(f,\cd) \sim \ca}\norm{\nabla_{\bw} L_{f(\cd)}(h_{\bw})}^2_2 \le B^2 N \Var(\ca)
\]
\end{lemma}
\begin{proof}
Denote $\phi_i(\x) = \frac{1}{B}\frac{\partial}{\partial w_i} h_{\bw}(\x)$ and note that $\phi_i(\x) \in [-1,1]$. Note that since $h_{\bw}(\x) \in [-1,1]$ for every $\x$, we have, for every $(f,\cd) \in \ca$: 
\[
L_{f(\cd)}(h_{\bw}) = \E_{\x \sim \cd} \ell(h_{\bw}(\x),f(\x))
= 1 - \E_{\x \sim \cd} h_{\bw}(\x)f(\x)
\]
where we use the fact that the hinge-loss satisfies $\ell(\hat{y},y) = 1-\hat{y} y$ for every $\hat{y} \in [-1,1]$.
Therefore, we get:
\begin{align*}
\E_{(f,\cd) \sim \ca} \norm{\nabla_{\bw} L_{f(\cd)}(h_{\bw})}^2_2
&= \E_{(f,\cd) \sim \ca} \sum_{i=1}^N\left( \frac{\partial}{\partial w_i} L_{f(\cd)}(h_{\bw}) \right)^2 \\
&= \sum_{i=1}^N \E_{(f,\cd) \sim \ca} \left(\E_{\x \sim \cd} f(\x)\frac{\partial}{\partial w_i} h_{\bw}(\x)\right)^2 \\
&= \sum_{i=1}^N \E_{(f,\cd) \sim \ca} B^2 \inner{f,\phi_i}^2_\cd
\le N B^2 \Var(\ca)
\end{align*}
\end{proof}

Using the above, we can show that running gradient-descent with $\Delta$-approximate gradients has high loss on average, unless $\Delta$ is very small:
\begin{theorem}
\label{thm:hardness_gd}
Assume we initialize $\bw_0$ from some distribution $\cw$ such that almost surely, for every $\x \in \cx$ we have $\abs{h_{\bw_0}(\x)} \le 1$ and $\norm{\nabla_{\bw} h_{\bw_0}(\x)} \le B$ for some $B > 0$.
Then, if $\Delta \ge 4\sqrt{2B^2N \Var(\ca)}$, there exists some $(f, \cd) \in \ca$ such that $\Delta$-approximate gradient-descent returns a hypothesis $h_{\bw_T}$ which satisfies:
\[
\E_{\bw_0 \sim \cw} L_{f(\cd)}(h_{\bw_T}) \ge \frac{3}{4} \left( 1- \sqrt{8\Var(\ca)} \right)
\]
\end{theorem}
%
\begin{proof}
Fix some $\bw_0$ satisfying $\abs{h_{\bw_0}(\x)} \le 1$ and $\norm{\nabla_{\bw} h_{\bw_0}(\x)} \le B$ for every $\x \in \cx$.
So, from Lemma \ref{lem:small_gradient} we have:
\begin{align*}
\E_{(f,\cd) \sim \ca} \norm{\nabla_{\bw} L_{f(\cd)}(h_{\bw_0})}^2_2
\le N B^2 \Var(\ca)
\end{align*}
From Markov's inequality, we get that with probability at least $1-\frac{4B^2 N \Var(\ca)}{\Delta^2}$ over the choice of $(f,\cd) \sim \ca$ we have $\norm{\nabla_\bw L_{f(\cd)}(h_{\bw_0})}_2 < \frac{\Delta}{2}$, and in this case $\bv_t = 0$. So, for every $(f,\cd) \in \ca$ with $\norm{\nabla_\bw L_{f(\cd)}(h_{\bw_0})}_2 < \frac{\Delta}{2}$, gradient-descent will return $h_{\bw_0}$. Since $h_{\bw_0}(\x) \in [-1,1]$, from Lemma \ref{lem:corr_count}, with probability at least $7/8$ over the choice of $(f,\cd) \sim \ca$ we have:
\[
L_{f(\cd)}(h_{\bw_0}) = 1-\inner{f,h_{\widetilde{\bw}_0}}_\cd \ge 1- \sqrt{8 \Var(\ca)}
\]
Using the union bound, with probability at least $\frac{7}{8}-\frac{4B^2N \Var(\ca)}{\Delta^2} \ge \frac{3}{4}$ over the choice of $(f,\cd)\sim \ca$, gradient-descent algorithm will return a hypothesis with loss at least $1-\sqrt{8\Var(\ca)}$, and so:
\[
\E_{(f,\cd)\sim \ca} L_{f(\cd)} (h_{\bw_T}) \ge \frac{3}{4} \left( 1- \sqrt{8\Var(\ca)} \right)
\]
Applying the above for a random choice of $\bw_0 \sim \cw$ we get:
\begin{align*}
\E_{(f,\cd)\sim \ca} \E_{\bw_0 \sim \cw}  L_{f(\cd)} (h_{\bw_T})= \E_{\bw_0 \sim \cw} \E_{(f,\cd)\sim \ca} L_{f(\cd)} (h_{\bw_T}) \ge \frac{3}{4} \left( 1- \sqrt{8\Var(\ca)} \right)
\end{align*}
And so the required follows
\end{proof}

So far, we analyzed the gradient-descent algorithm with respect to an approximation of the population gradient. The above result shows that if the initial gradient is very small, then gradient-descent is ``stuck'' on the first iteration. In practice, however, gradient-descent uses stochastic estimation of the population gradient. In this case, the stochastic noise due to the gradient estimation will cause non-zero update steps, even if the population gradient is zero. To account for this setting, we consider a noisy version of gradient-descent. The \textit{$\sigma$-noisy gradient-descent} performs the same update as in (\ref{eq:approx_gd}), with $\bv_t \in \Delta \integers^N$ which satisfies: 
\[
\norm{\bv_t - \left(\nabla_\bw L_{f(\cd)}(h_{\bw_{t-1}}) + \xi_t\right)}_\infty \le \frac{\Delta}{2}
\]
where $\xi_1, \dots, \xi_T$ are i.i.d. random noise variables with $\xi_t \in \Delta \integers^N$ and $\norm{\xi_t}_2 \le \sigma$. For the \textit{noisy gradient-descent} algorithm, we get the following hardness result:

\begin{theorem}
\label{thm:hardness_noisy_gd}
Assume we initialize $\bw_0$ from some distribution $\cw$ such that almost surely, for every $\x \in \cx$ we have $\abs{h_{\bw_0}(\x)} \le \frac{1}{2}$. Assume that every $h_\bw \in \ch$ is differentiable and satisfies $\norm{\nabla_\bw h_\bw(\x)} \le B$ for all $\x \in \cx$.
Then, there exists some $(f, \cd) \in \ca$ such that running the noisy gradient-descent algorithm for at most $T \le \frac{\Delta^2}{32B^2 N \Var(\ca)}$ steps with $\eta \le \frac{1}{2\sigma B T}$, returns a hypothesis $h_{\bw_T}$ satisfying:
\[
\E_{\bw_0,\xi_1,\dots, \xi_T} L_{f(\cd)}(h_{\bw_T}) \ge \frac{3}{4} \left(1- \sqrt{8 \Var(\ca)}\right)
\]
\end{theorem}

\begin{proof}
Fix some $\xi_1, \dots, \xi_T \in \Delta \integers^N$ such that $\norm{\xi_t}_2 \le \sigma$ for every $t \in [T]$. Fix some $\bw_0 \in \Delta \integers^N$ with $\norm{h_{\bw_0}}_{\infty} \le \frac{1}{2}$.
We define $\widetilde{\bw}_t := \bw_0 + \eta\sum_{i=1}^t \xi_i$, and observe that for every $t$, since $\eta \le \frac{1}{2\sigma BT }$, we have: 
\[
\abs{h_{\widetilde{\bw}_t}(\x)} \le \abs{h_{\bw_0}(\x)} + \abs{h_{\widetilde{\bw}_t}(\x) - h_{\bw_0}(\x)} \le \frac{1}{2} + \eta B \norm{\sum_{i=1}^t \xi_i}_2 \le 1
\]
Therefore, using Lemma \ref{lem:small_gradient} we get:
\[
\E_{(f,\cd) \sim \ca} \frac{1}{T} \sum_{t=0}^{T-1} \norm{\nabla_{\bw} L_{f(\cd)}(h_{\widetilde{\bw}_t})}^2_2 = \E_{t \sim [T-1]} \E_{(f,\cd) \sim \ca}\norm{\nabla_{\bw} L_{f(\cd)}(h_{\widetilde{\bw}_t})}^2_2 \le B^2 N \Var(\ca)
\]
From Markov's inequality, with probability at least $7/8$ over the choice of $(f,\cd) \sim \ca$, we have $\sum_{t=0}^{T-1} \norm{L_{f(\cd)}(h_{\widetilde{\bw}_t})}^2_2 < 8 T B^2 N \Var(\ca) \le \frac{\Delta^2}{4}$. For every such $(f,\cd)$, we have $\norm{L_{f(\cd)}(h_{\widetilde{\bw}_t})}_2 < \frac{\Delta}{2}$ for every $t \in [T-1]$, and so $\bv_t = \xi_t$ for every $t$, in which case we have the updates $\bw_t = \widetilde{\bw}_t$, and the noisy gradient-descent algorithm will output $h_{\widetilde{\bw}_T}$.
Since $h_{\widetilde{\bw}_T}(\x) \in [-1,1]$, from Lemma \ref{lem:corr_count}, with probability at least $7/8$ over the choice of $(f,\cd) \sim \ca$ we have:
\[
L_{f(\cd)}(h_{\widetilde{\bw}_T}) = 1-\inner{f,h_{\widetilde{\bw}_T}}_\cd \ge 1- \sqrt{8 \Var(\ca)}
\]

All in all, using the union bound, w.p. at least $3/4$ over the choice of $(f,\cd) \sim \ca$ the gradient-decent algorithm returns a hypothesis $h_{\bw_T}$ with loss $\ge 1-\sqrt{8\Var(\ca)}$, and so:
\[
\E_{(f,\cd) \sim \ca} L_{f(\cd)}(h_{\bw_T}) \ge \frac{3}{4}\left(1-\sqrt{8\Var(\ca)}\right)
\]
Now, for a random choice of $\bw_0,\xi_1, \dots, \xi_T$ we have:
\[
\E_{(f,\cd) \sim \ca} \E_{\bw_0,\xi_1, \dots, \xi_T} L_{f(\cd)}(h_{\bw_T}) = \E_{\bw_0,\xi_1, \dots, \xi_T} \E_{(f,\cd) \sim \ca} L_{f(\cd)}(h_{\bw_T}) \ge \frac{3}{4}\left(1-\sqrt{8\Var(\ca)}\right)
\]
and therefore the required follows.
\end{proof}

Applying the previous theorem for the family of uniform parities, we get that gradient-descent fails to reach non-trivial loss on the class of parities, unless the approximation tolerance is exponentially small or the number of steps is exponentially large. This result is similar to the results of \cite{shalev2017failures} and \cite{abbe2018provable}.

\section{General Distribution Families}

In the previous section, we showed various hardness of learning and approximation results, all derived from the measure $\Var(\ca)$. We showed the application of such hardness results to the case of parities, or more generally --- families of orthogonal functions. However, note that the measure $\Var(\ca)$ can be applied to any family of distributions, and therefore all of our results can be derived for the very general setting of learning arbitrary families of labeled distributions. In this section, we interpret these results for general distribution families, and show how to derive bounds on $\Var(\ca)$ in the general case. Using this, we show novel results on hardness of approximation and learnability of DNFs and $\textrm{AC}^0$ circuits.

We start by showing a general method for bounding $\Var(\ca)$.
Let $M(\ca)$ be the linear operator from $\reals^{\cx}$ to $\reals^{\ca}$, such that for every $\phi : \cx \to \reals$, $M(\ca)(\phi)_{(f,\cd)} = \inner{f,\phi}_\cd$. The linearity of $M(\ca)$ follows from the bi-linearity of the inner product $\inner{\cdot, \cdot}_\cd$. Note that when $\cx$ and $\ca$ are finite (as we assume in this work), $M(\ca)$ can be written in a matrix form, where $M(\ca) \in \reals^{\ca \times \cx}$ and $M(\ca)_{(f,\cd),\x} = f(\x) \cd(\x)$. In this case, we identify $\phi : \cx \to \reals$ with a vector $\bv(\phi) \in \reals^{\cx}$ with $\bv(\phi)_\x = \phi(\x)$, and we get:
\[
M(\ca) \bv(\phi) = [\sum_{\x \in \cx} f(\x) \phi(\x) \cd(\x)]_{(f,\cd)} = [\inner{f,\phi}_\cd]_{(f,\cd)}
\]
Now, observe that for every $\phi : \cx \to [-1,1]$ we have:
\begin{equation}
\label{eq:op_norm_bound}
\Var(\ca, \phi) = \Mean{(f,\cd) \sim \ca}{\inner{f,\phi}^2_\cd}
= \frac{1}{\abs{\ca}} \norm{M(\ca) \phi}^2_2
\le \frac{1}{\abs{\ca}} \norm{M(\ca)}_2^2 \norm{\phi}_2^2
\le \frac{\abs{\cx}}{\abs{\ca}} \norm{M(\ca)}^2_2
\end{equation}
Where $\norm{M(\ca)}_2$ is the $L_2$ operator norm of $M(\ca)$. Hence, Eq. (\ref{eq:op_norm_bound}) gives a general bound for $\Var(\ca)$, in terms of the operator norm of the matrix $M(\ca)$.

\subsection{Operator Norm of the AND-OR-AND Function}

In this part, we give a concrete family of distributions, generated by an AND-OR-AND type function, and analyze its variance. Specifically, we show that its variance decays like $2^{-O(n^{1/3})}$. The key for showing this result is bounding the operator norm of the relevant matrix, along with the analysis introduced in \cite{razborov2010sign}, which bounds the norm of a similar matrix. The main result is the following:

\begin{theorem}
\label{thm:and_or_variance}
For large enough $m$, there exist subspaces $\cx,\cz \subseteq \{\pm 1\}^{dm^3}$, for some universal constant $d > 0$, and a family $\ca$ over $\cx$ such that:
\begin{itemize}
\item For each $(f,\cd) \in \ca$, the function $f$ is of the form $f(\x) = \bigwedge_{i=1}^m \bigvee_{j=1}^{dm^2} (x_{ij} \wedge z_{ij})$, for some $\z \in \cz$.
\item $\Var(\ca) \le 17^{-2m}$.
\end{itemize}
\end{theorem}

For the proof of the theorem, we need the following simple result:

\begin{definition}
Let $\ca$ be some distribution family over an input space $\cx$, and let $\ca'$ be some distribution family over another input space $\cx'$. $\ca$ and $\ca'$ are \textbf{isomorphic} if there exists a bijection $\Psi : \cx \to \cx'$ such that $\ca = \{(f \circ \Psi, \cd \circ \Psi) ~:~ (f,\cd) \in \ca'_n\}$.
\end{definition}

\begin{lemma}
\label{lem:isomorphism}
If $\ca$ and $\ca'$ are isomorphic distribution families, then $\Var(\ca) = \Var(\ca')$.
\end{lemma}

\begin{proof}
Fix $\phi : \cx \to [-1,1]$, and observe that:
\begin{align*}
\Var(\ca,\phi) &= \E_{(f,\cd) \sim \ca}{\inner{f,\phi}_\cd^2} 
= \E_{(f,\cd) \sim \ca'_n}{\inner{f \circ \Psi,\phi}_{\cd \circ \Psi}^2} 
= \E_{(f,\cd) \sim \ca'_n}{\left( \Mean{\x \sim \cd \circ \Psi}{f(\Psi (\x)) \phi(\x)} \right)^2} \\
&= \E_{(f,\cd) \sim \ca'_n}{\left( \Mean{\x \sim \cd}{f(\x) \phi(\Psi^{-1}(\x))} \right)^2} = \Var(\ca',\phi \circ \Psi^{-1})
\le \Var(\ca')
\end{align*}
Therefore $\Var(\ca) \le \Var(\ca')$, and the required follows from symmetry.
\end{proof}

Now, the key for bounding the operator norm related to the family $\ca$, is using the pattern-matrix technique, as used in \cite{razborov2010sign}. 
\paragraph{Pattern Matrix}

Let $n , N$ be two integers, and let $\cv(N,n)$ be the family of subsets $V \subset [N]$ of size $\abs{V} = n$, with one element from each block of size $N/n$ from $[N]$. Define the projection onto $V$ by $x|_V = (x_{i_1}, \dots, x_{i_n})$ where $i_1 < \dots < i_n$ are the elements of $V$.

\begin{definition}
For $\phi : \{\pm 1\}^n \to \reals$, the $(N,n,\phi)$-pattern matrix is the matrix:
\[
A = [\phi(\x|_V \oplus \bw)]_{\x \in \{\pm 1\}^N, (V,\bw) \in \cv(N,n) \times \{\pm 1\}^n}
\]
\end{definition}

\begin{theorem} (\cite{razborov2010sign})
\label{thm:pattern_mat}
Let $n = 4m^3$ and $N=17^6n$.
Let $MP_m(\x) = \bigwedge_{i=1}^m \bigvee_{j=1}^{4m^2} x_{ij}$, and let $M$ be the $(N,n,MP_m)$-pattern matrix. There exists a distribution $\mu : \{\pm 1\}^n \to \reals_+$, such that the $(N,n,\mu)$-pattern matrix $P$ satisfies $\norm{M \circ P} \le 17^{-m} 2^{-n} \sqrt{2^{N+n} (\frac{N}{n})^n}$ (where $\circ$ denotes the Hadamard product).
\end{theorem}

\begin{proof} of Theorem \ref{thm:and_or_variance}.

Let $\mu : \{\pm 1\}^n \to \reals$ be the distribution from Thm \ref{thm:pattern_mat}. Denote $\cx' = \{\pm 1\}^N$, $\cz' = \cv(N,n) \times \{\pm 1\}^n$. Fix some $(V,\bw) \in \cz'$ denote $f_{V,\bw}(\x) = MP_m(\x|_V \oplus \bw)$, and  define the distribution $\cd_{V,w}$ over $\cx'$ such that $\cd_{V,w}(\x) = \frac{1}{2^{N-n}} \mu(\x|_V \oplus \bw)$. $\cd_{V,\bw}$ indeed defines a distribution over $\cx'$, since:
\[
\sum_{\x \in \cx'} \cd_{V,\bw}(\x) = \sum_{\z \in \{\pm 1\}^n} \sum_{\x \in \cx', \x |_V = \z} \frac{1}{2^{N-n}} \mu (\z \oplus \bw) = \sum_{\z \in \{\pm 1\}^n} \mu(\z \oplus \bw) = 1
\]
 Define the family $\ca' = \{(f_{V,\bw}, \cd_{V,\bw}) ~:~ (V,\bw) \in \cz'\}$ and recall that we defined $M(\ca') = [f_{V,\bw}(\x) \cd_{V,\bw}(\x)]_{\x \in \cx', (V,\bw) \in \cz}$. Let $M$ be the $(N,n,MP_m)$-pattern matrix and let $P$ bet the $(N,n,\mu)$-pattern matrix, and so $M(\ca') = \frac{1}{2^{N-n}} M \odot P$, and from Thm \ref{thm:pattern_mat} we have:
\[
\norm{M(\ca')} = 2^{n-N} \norm{M \odot P} \le 2^{n-N} 17^{-m} 2^{-n} \sqrt{2^{N+n} (\frac{N}{n})^n} =  17^{-m} \sqrt{2^{n-N} (\frac{N}{n})^n}
\]
And from Eq. \ref{eq:op_norm_bound} we get:
\[
\Var(\ca') \le \frac{\abs{\cx'}}{\abs{\ca'}} \norm{M(\ca')}^2
= \frac{2^N}{(N/n)^n2^n} 17^{-2m} 2^{n-N}(N/n)^n = 17^{-2m}
\]

Now, we identify $\cx'$ and $\cz'$ with subsets of $\{\pm 1\}^{n'}$, for $n' = m \cdot 4m^2 \cdot N/n \cdot 2 = 8m^3 \frac{N}{n}$. Denote $\Psi : \cx' \to \{\pm 1\}^{n'}$ such that $\Psi(\x)_{ijk\epsilon} = \ind\{x_{ij} = \epsilon\}$, and denote $\Phi : \cz' \to \{\pm 1\}^{n'}$ such that $\Phi(V,\bw)_{ijk\epsilon} = \ind\{w_{ij} \ne \epsilon\} \vee \ind \{V_{ij} = k\}$, where $V_{ij} \in [N/n]$ indicates which elements from the $ij$ block was selected by $V$. Now, note that:
\begin{align*}
f_{V,\bw}(\x) &= \bigwedge_{i=1}^m \bigvee_{j=1}^{m^2} (\x|_V \oplus \bw)_{i,j} = \bigwedge_{i=1}^m \bigvee_{j=1}^{4m^2} \bigvee_{k=1}^{N/n} \bigvee_{\epsilon \in \{\pm 1\}} ((x_{ijk} = \epsilon) \wedge (w_{ij} \ne \epsilon) \wedge (v_{ij} = \epsilon)) \\
&= \bigwedge_{i=1}^m \bigvee_{j=1}^{4m^2} \bigvee_{k=1}^{N/n} \bigvee_{\epsilon \in \{\pm 1\}} (\Psi(\x)_{ijk\epsilon} \wedge \Phi(V,\bw)_{ijk\epsilon})
\end{align*}
Finally, we define $\cx = \Psi(\cx')$ and $\cz = \Phi(\cz')$, and define $\ca = \{(\cd_{\Psi^{-1}(\z)}, f_{\Psi^{-1}(\z)}) ~:~ \z \in \cz\}$. Observe that $\ca$ and $\ca'$ are isomorphic, and therefore from Lemma \ref{lem:isomorphism} we get $\Var(\ca) = \Var(\ca') \le 17^{-2m}$.
\end{proof}

In the rest of this section, we show how Theorem \ref{thm:and_or_variance} can be used to derive hardness results on approximation and learnability of DNFs and $\mathrm{AC}^0$.

\subsection{Hardness of Approximating DNFs using Linear Classes}

Observe the family $\ca$ as defined in Theorem \ref{thm:and_or_variance}. Note that for every $(f,\cd) \in \ca$, the function $\neg f$ is in fact a DNF. Using the fact that $\Var(\ca) \le 17^{-2m}$, along with Eq. (\ref{eq:lin_approx_lowerbound}), we get that for every mapping $\Psi : \cx \to [-1,1]^N$ we have:
\[
\max_{(f,\cd) \in \ca} \min_{h \in \ch_\Psi} L_{f(\cd)}(h) \ge \ell_0-B\sqrt{N} 17^{-m}
\]
Therefore, the following result is immediate:
\begin{corollary}
Let $\ell$ be some convex loss function, satisfying $\ell(0,y) = \ell_0$ and $\ell'(0,y) = -y$.
For every mapping $\Psi : \{\pm 1\}^n \to [-1,1]^N$, there exists some DNF $f$ and a distribution $\cd$ over $\{\pm 1\}^n$ such that for every $\bw \in \reals^N$, the function $h_\bw(\x) = \inner{\bw,\Psi(\x)}$ has loss lower-bounded by:
\[
L_\cd(h_\bw) \ge \ell_0 - \frac{\norm{\bw} \sqrt{N}}{2^{\Omega(n^{1/3})}}
\]
\end{corollary}

Note that DNFs are known to be learnable using polynomial threshold functions of degree $\tilde{O}(n^{1/3})$, which can be implemented using a linear classifier over a mapping $\Psi : \{\pm 1\}^n \to [-1,1]^{N}$, for $N = 2^{\tilde{O}(n^{1/3})}$ \citep{klivans2004learning}. Our result shows that no linear class can approximate DNFs unless $N = 2^{\Omega(n^{1/3})}$, regardless of the choice of $\Psi$. This extends the result in \cite{razborov2010sign}, which shows a similar bound on the dimension of a linear class that is required to exactly express DNFs.

\subsection{Hardness of Approximating $\mathrm{AC}^0$ using Shallow Networks}

In Theorem \ref{thm:and_or_variance} we showed a family $\ca$ over $\cx \subseteq \{\pm 1\}^n$, where every $(f,\cd) \in \ca$ is identified with $\z \in \cz$, with $\cz \subseteq\{\pm1\}^n$, such that $f(\x) = \bigwedge_{i=1}^m \bigvee_{j=1}^{dm^2} (x_{ij} \wedge z_{ij})$. So, we can define the function $F : \cx \times \cz \to \{\pm 1\}$, induced from the family $\ca$, such that $F(\x,\z) = \bigwedge_{i=1}^m \bigvee_{j=1}^{dm^2} (x_{ij} \wedge z_{ij})$.
Observe that $F \in \mathrm{AC}^0$, where $\mathrm{AC}^0$ is the class of polynomial-size constant-depth AND/OR circuits. Then, Theorem \ref{thm:shallow_approx_lowerbound} implies the following:
\begin{corollary}
Let $\cx = \{\pm 1\}^n$, and let $\ell$ be a 1-Lipschitz convex loss satisfying $\ell(0,y) = \ell_0$ and $\ell'(0,y) = -y$. There exists a function $F : \cx \to [-1,1]$ such that $F \in \mathrm{AC}^0$, and a distribution $\cd'$ over $\cx$, such that for every neural-network $g$, with $1$-Lipschitz activation, using $k$ neurons with weights of $L_2$-norm at most $R$, has loss lower-bounded by:
\[
L_{F(\cd')}(g) \ge \ell_0 - \frac{\sqrt{k} R^2 n^{5/6}}{2^{\Omega(n^{1/3})}}
\]
\end{corollary}

This extends the result in \cite{razborov2010sign}, which shows that such function cannot be exactly implemented using a threshold circuit, unless its size is $2^{\Omega(n^{1/3})}$.

\subsection{Hardness of Learning DNFs}
As noted, the family $\ca$ from Theorem \ref{thm:and_or_variance} defines a family of distributions labeled by DNF formulas. In Theorem \ref{thm:sq_hardness}, we showed that families with small variance are hard to learn from correlation queries. Therefore, we get an exponential lower bound on the number of queries required for learning DNF formulas from correlation queries, with respect to the hinge-loss:
\begin{corollary}
For any $\tau > 0$, any statistical-query algorithm that makes only correlation queries needs at least $\tau^2 2^{\Omega(n^{1/3})}$ queries to achieve hinge-loss $< 1-\tau$, square loss $< 1-2\tau$ or zero-one loss $<\frac{1}{2}-\frac{1}{2}\tau$, on DNF formulas of dimension $n$.
\end{corollary}

Following these results, using Theorem \ref{thm:hardness_noisy_gd} shows that any hypothesis class (with bounded gradients), optimized with gradient-descent on the hinge-loss, will need at least $\Omega(n^{1/3})$ gradient-iterations to approximate the family of DNF formulas:

\begin{corollary}
Assume we initialize $\bw_0$ from some distribution $\cw$ such that almost surely, for every $\x \in \cx$ we have $\abs{h_{\bw_0}(\x)} \le \frac{1}{2}$. Assume that every $h_\bw \in \ch$ is differentiable and satisfies $\norm{\nabla_\bw h_\bw(\x)} \le B$ for all $\x \in \cx$.
Then, there exists some DNF $f$ and distribution $\cd$ such that the noisy gradient-descent algorithm with the hinge-loss requires at least $2^{\Omega(n^{1/3})} \frac{\Delta^2}{B^2N}$ steps to approximate $f$ with respect to $\cd$.
\end{corollary}

Note that these hardness results match the currently known upper bound of learning DNFs in time $2^{\tilde{O}(n^{1/3})}$, due to \cite{klivans2004learning}. While these results apply only to a restricted family of algorithms (namely, correlation query algorithms and gradient-descent), we hope similar techniques can be used to show such hardness results for a broader family of algorithms.

\newpage

\bibliography{bib}
\bibliographystyle{plain}

\end{document}